\documentclass[12pt]{article}

\usepackage{url}
\usepackage{amsmath,amssymb}
\usepackage[dvips]{graphicx,color}
\usepackage[round]{natbib}
\usepackage{mathrsfs}
\usepackage{bm, xcolor}
\usepackage{mathrsfs}
\usepackage{verbatim}
\usepackage{threeparttable}
\usepackage{color}
\usepackage[utf8]{inputenc}
\usepackage{booktabs} 
\usepackage{graphicx}
\usepackage{subfigure}
\usepackage{float}
\usepackage{listings}
\usepackage{amssymb}
\usepackage{amsthm}
\usepackage{multirow}
\usepackage{caption}
\usepackage{algorithm}
\usepackage{algpseudocode}

\renewcommand{\baselinestretch}{1.3}

\makeatletter
\def\singlespace{\def\baselinestretch{1}\@normalsize}

\newtheorem{theorem}{Theorem}

\renewcommand{\hat}{\widehat}

\def\singlespace{\def\baselinestretch{1}\@normalsize}

\makeatother

\newdimen\biblioindent    \biblioindent=30pt

 at 7truept

\def\beq{\begin{equation}}
\def\eeq{\end{equation}}
\def\beqn{\begin{eqnarray}}
\def\eeqn{\end{eqnarray}}
\def\beqnn{\begin{eqnarray*}}
\def\eeqnn{\end{eqnarray*}}

\begin{document}

\title{Enhancing Robustness of Gradient-Boosted Decision Trees through One-Hot Encoding and Regularization\thanks{The views expressed in the paper are those of the authors and do not represent the views of Wells Fargo.}}

\author{Shijie Cui$^a$, Agus Sudjianto$^a$, Aijun Zhang$^a$ and Runze Li$^b$\\
  $^a$Corporate Model Risk, Wells Fargo, Charlotte, NC 28202, USA\\
\{shijie.cui, agus.sudjianto, aijun.zhang\}@wellsfargo.com\\
$^b$Department of Statistics, Pennsylvania State University\\
University Park, PA 16802, USA\\
rzli@psu.edu\\}
\date{}
\maketitle
\begin{abstract}
Gradient-boosted decision trees (GBDT) are widely used and highly effective machine learning approach for tabular data modeling. However, their complex structure may lead to low robustness against small covariate perturbation in unseen data. In this study, we apply one-hot encoding to convert a GBDT model into a linear framework, through encoding of each tree leaf to one dummy variable. This allows for the use of linear regression techniques, plus a novel risk decomposition for assessing the robustness of a GBDT model against covariate perturbations. We propose to enhance the robustness of GBDT models by refitting their linear regression forms with $L_1$ or $L_2$ regularization. Theoretical results are obtained about the effect of regularization on the model performance and robustness. It is demonstrated through numerical experiments that the proposed regularization approach can enhance the robustness of the one-hot-encoded GBDT models.
\end{abstract}

\section{Introduction}
Gradient-boosted decision trees (GBDT) models have gained increasing popularity in a variety of business applications due to its better predictive performance compared to traditional statistical models. One reason for this is the flexibility of GBDT in capturing non-linear relationships between dependent and independent variables, as well as interactions among different independent variables. GBDT is a special case of ensemble models, which involve training multiple models and combining their predictions to obtain a final prediction. The ability to model complex relationships and interactions makes GBDT  a powerful tool for a wide range of predictive modeling tasks.

The final model produced by GBDT is often considered a “black box” model because the combination of decision trees leads to a complex model structure that is difficult for humans to understand. In general, the capacity of the model to fit the training data increases as its model complexity increases.
Traditional statistical models encounter overfitting problems when models are too complex. However, in modern machine learning, evidence suggests that some models can still predict well even with a perfect fit to training data, leading to the construction of increasingly complex models. This phenomenon has been referred to as ``benign overfitting" and has been discussed in the literature \citep{bartlett2020benign, tsigler2020benign, shamir2022implicit}.

Is benign overfitting really benign? In machine learning, models are typically evaluated based on their accuracy, such as test mean squared error for regression tasks or AUC for classification tasks.  However, another crucial property of machine learning models is their robustness against small perturbations in the input data. Unfortunately, some existing works have shown that many popular machine learning models are susceptible to adversarial attacks if not designed carefully, highlighting the importance of robustness.
Researchers find evidence that with increased capacity comes the risk of overfitting,  the model becomes highly sensitive to noise or random variations in the data. For instance, studies by \citet{globerson2006nightmare}, \citet{szegedy2013intriguing}, \citet{goodfellow2014explaining}, \citet{andriushchenko2019provably}, \citet{carlini2017towards} and \citet{javanmard2020precise}  have demonstrated that deep learning models can be vulnerable to adversarial examples, where small changes to the data can lead to incorrect predictions.  In this paper, we focus on the robustness of machine learning models (in terms of prediction rather than parameter estimation), defined as their ability to persist noise, and differentiate this concept from the robustness analysis of estimators in traditional statistics.

For some commonly used interpretable models such as linear or generalized linear models, it is often easy to evaluate the robustness because we can trace the impact of features to the response due to interpreatbility of these models. However, the simplicity of such models can limit their capability in modeling complex relationship between the response and the predictors in real-world applications. Complex machine learning models like tree ensembles are often used to gain better predictive performance. There are existing works for evaluating the robustness of machine learning models
\citep{nicolae2018adversarial, carlini2019evaluating, chen2019robustness}. However, most of them only focus on empirical analyses and intuitive observations. There lacks a systematic approach for evaluating robustness of complex machine learning models. This paper aims to bridge the gap in robustness analysis between interpretable models and machine learning models, specifically the GBDT models and help develop robustness testing tools for machine learning models.  The robustness testing tools have been implemented in the open-access PiML toolbox by \citet{sudjianto2023piml}.

We focus on the importance of model robustness evaluation and highlight the potential of one-hot encoding with regularization in developing more reliable and accurate machine learning models. This paper makes the following contributions to the machine learning literature.
\begin{itemize}
\item We demonstrate that the so-called benign overfitting GBDT models may not be robust to data perturbations, through numerical examples and traditional bias-variance decomposition.

\item We develop a novel risk decomposition method to analyze the robustness of GBDT models by applying the one-hot encoding technique.

\item We propose to enhance the robustness of GBDT models by refitting their linear regression forms with $L_1$ or $L_2$ regularization, and provide theoretical analysis of the regularization effectiveness on model robustness and performance.
\item We show through numerical studies that the proposed approach can improve the robustness of GBDT models against covariate perturbations in testing data.
\end{itemize}

The rest of this paper is organized as follows. In Section \ref{sec202303271434}, we evaluate the robustness of the GBDT model and demonstrate that a complex GBDT model may not be robust with respect to data perturbation. In Section \ref{sec202302161144}, we propose GBDT OHE (One-hot-encoded GBDT) and show how a GBDT model can be expressed as a linear model using GBDT OHE, and propose a novel risk decomposition method when the model has perturbations to analyze robustness of a GBDT model. In Section \ref{sec202303271438}, we integrate regularization and the GBDT OHE techniques, and study its impact on model robustness and performance theoretically. In Section \ref{sec202303131258},  we present real data analyses to showcase the performance of our proposed approach.  In Section \ref{sec202303291157} we conclude our results and give potential future work.
Appendix I presents an introduction of data sets we used for real data analysis.
Appendix II consists of the final tuning parameter results of different models fitted in Section~\ref{sec202303131258}.

\section{GBDT Model Robustness}\label{sec202303271434}
We begin with an illustration of the robustness weakness of traditional GBDT models, using bias and variance decomposition and some numerical examples. 
Bias and variance are two important sources of model performance in machine learning. Here and hereafter bias refers to the systematic error that arises when a model is unable to capture the true relationship between input variables and the output variable, and variance refers to the variance of error that arises due to model sensitivity to variations in the training data. 

Consider a regression model 
$$
y = f(x)+\varepsilon,
$$ where $\varepsilon \sim P_\varepsilon$ is a noise term with zero mean and variance $\sigma^2_\varepsilon.$ Suppose $E_\mathcal{D}$ is the expectation on training data and $E_\mathcal{X}$ is the expectation of new observations or testing data (i.e. data not used for training model). To evaluate the performance of a fitted model $\hat{f}$, we decompose the risk as
\begin{eqnarray}
&&E_\mathcal{X}E_\mathcal{D}(y-\hat{f}(x))^2\nonumber\\
&=&E_\mathcal{X}[f(x)-E_{\mathcal{D}}(\hat{f}(x))]^2+E_\mathcal{X}E_\mathcal{D}[E_{\mathcal{D}}(\hat{f}(x))-\hat{f}(x)]^2+\sigma_\varepsilon^2\\
&=:&\text{Bias}^2 + \text{Variance} +\text{Irreducible error},
\end{eqnarray}
where $E_\mathcal{X}[f(x)-E_{\mathcal{D}}(\hat{f}(x))]^2$ is the square of model bias measuring how the fitted model deviates from the true model,  $E_\mathcal{X}E_\mathcal{D}[E_{\mathcal{D}}(\hat{f}(x))-\hat{f}(x)]^2$ is the model variance measuring  the sensitivity of the fitted model when the training data changes, and the irreducible error $\sigma_\varepsilon^2$ represents the minimum achievable error that cannot be reduced by any model. This error arises from the noise in the data and cannot be eliminated by any modeling technique. The bias-variance trade-off describes the balance between bias and variance that machine learning models traditionally strike to perform well.  Models that are too simple may underfit the data, resulting in large bias and large errors on both training and testing data. Conversely, models that are too complex may overfit the data, leading to low bias and low errors on the training data but high errors on the testing data.  Traditionally, in statistical learning theory, it has been observed that the risk function tends to be U-shaped, with a minimum point that corresponds to the best trade-off between bias and variance. Finding a model that strikes a balance between bias and variance can be achieved through techniques such as regularization, cross-validation, and model selection. Regularization, for example, can be used to control model complexity to reduce variance, while cross-validation can be used to assess the generalization performance of models  under different regularization sizes and help select the proper amount of regularization. 
\begin{figure}
	\centering
	\subfigure{
		\begin{minipage}[b]{0.4\textwidth}
			\includegraphics[width=1\textwidth,height=7cm]{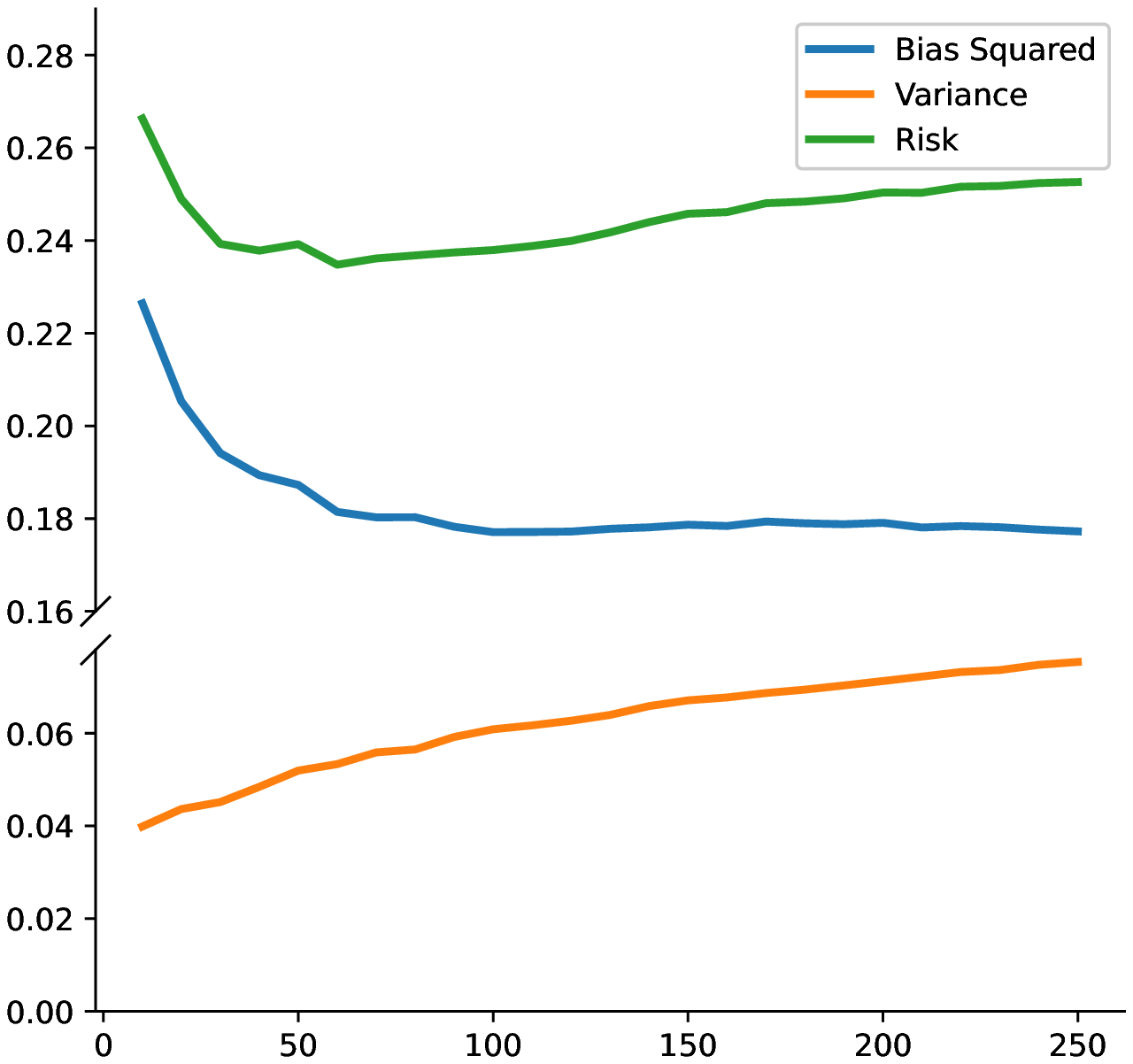}
		\end{minipage}
	   }
    	\subfigure{
    		\begin{minipage}[b]{0.4\textwidth}
   		 	\includegraphics[width=1\textwidth,height=7cm]{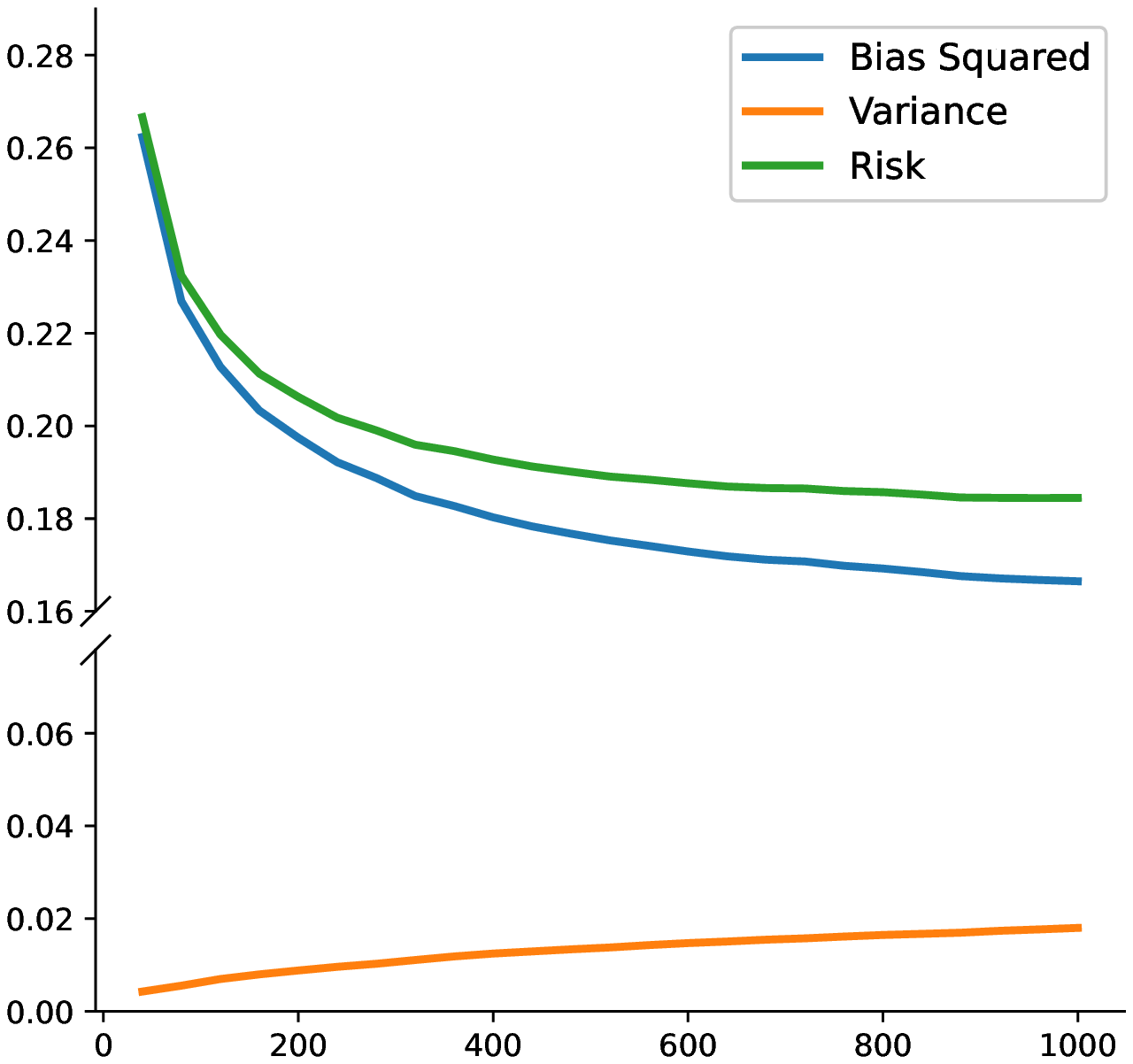}
    		\end{minipage}
    	}
    	\caption{Left panel depicts the risk function along with bias$^2$ and variance for large learning rate (i.e. learning rate equals 1), and the risk function is U-shaped. Right panel is for small learning rate (learning rate is 0.1). The risk function does not have U-shape.}\label{fig202303131328}
\end{figure}
In deep learning, it is not always the case that the risk function is a typical U-shaped as model complexity changes. In some cases, the variance curve may become flat or increase very slowly as the model becomes more and more complex, leading to ``benign overfitting". To understand this phenomenon, we fit XGBoost (eXtreme Gradient Boosting, proposed by \citet{chen2016xgboost}) models with different learning rates on the California House Pricing (CHP) dataset. The curves of risk function, 
squared bias and variance are depicted in Figure~\ref{fig202303131328} with learning rate
1.0 and 0.1. Figure~\ref{fig202303131328} shows that the XGBoost model with a large learning rate 1.0 has a U-shaped risk function, while the model with a small learning rate 0.1 does not. Note that in Figure \ref{fig202303131328},  we mix irreducible error with the squared bias term since it is not possible to separate them for real-world applications. 

With models becoming more and more complex, it is of great interest to understand how the robustness of the models changes.
To illustrate the impact of model complexity on robustness, we conduct a simple experiment where we add small perturbations to the testing data to evaluate the model robustness. Specifically, we perturb each predictor in the data by adding a normally distributed perturbation with a standard deviation be $5\%$ of the standard deviation of  original predictor. Figure \ref{fig202303131335} depicts the new risk function of XGBoost model  with small learning rate 0.1. When we add perturbations to the data, the U-shaped trend re-emerged, indicating that complex models are less robust to perturbation. 

It is natural to ask: 1) how to analyze the robustness of a GBDT model, and 2) how to enhance the robustness of a GBDT model? These two questions define our scope of study, and they will be addressed in subsequent sections of this paper.

\begin{figure}[t]
\centering
\includegraphics[width=0.5\textwidth,height=0.5\textwidth]{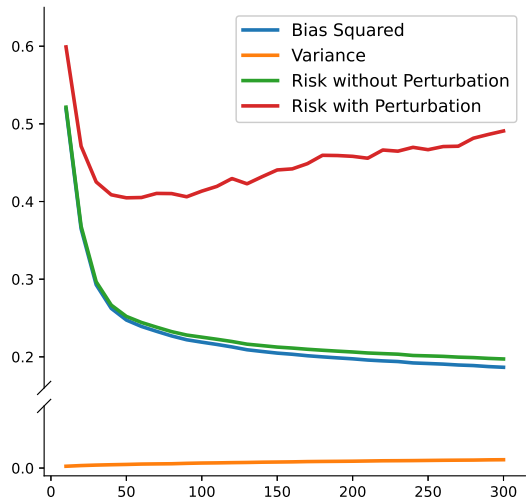}
\caption{U-shape when the learning rate is small but with perturbation.}\label{fig202303131335}
\end{figure}

\section{GBDT OHE and Risk Decomposition}\label{sec202302161144}

To analyze and enhance the robustness of a GBDT model, we propose to apply one-hot encoding to GBDT models, and name this technique as GBDT OHE.  It involves representing gradient-boosted decision trees as a linear model by encoding each tree leaf to one dummy variable. We propose a novel risk decomposition approach that includes a new perturbation term, which enable us to analyze the robustness of a GBDT model. 

\subsection{One-hot Encoding}

A GBDT model has a complicated structure as a whole, but the individual decision trees that make up the model are generally easy because each leaf rule is clear and understandable. It is remarkable that if we one-hot encode the leaves of a GBDT into dummy variables, the prediction of each GBDT model can be viewed as a linear combination of the leaf nodes. We refer to it as the GBDT OHE method. To see this, for a given data set with $n$ observations and $q$ predictors 
$\mathcal{D}=\{(\mathbf{x}_i,y_i ), |\mathcal{D}|=n,\mathbf{x}_i \in \mathcal{X} \subseteq\mathbb{R}^q, y_i\in \mathcal{Y}\subseteq\mathbb{R} \}$. Suppose that $\mathbf{x}_i$'s are identically independently generated from a random distribution $P_\mathcal{\mathbf{x}}$. Regression function of $y_i$ on $\mathbf{x}_i$ is assumed to be 
\begin{eqnarray}
y_i=f(\mathbf{x}_i)+\varepsilon_i,
\end{eqnarray} where $\varepsilon_i \sim P_\varepsilon$ is a noise term with zero mean and variance $\sigma^2_\varepsilon.$  $\varepsilon_i$ is assumed to be independent of $\mathbf{x}_i$.
Suppose that $y_i$ is a scalar response for problem with continuous response and suppose $(\mathbf{x}_i,y_i)\sim P_{\mathbf{x},y}.$ Further suppose $(\mathbf{x},y)\sim P_{\mathbf{x},y}$ is a new sample representing a testing data point.

A GBDT model assumes that the regression function is an ensemble of $M$ additive trees. GBDT model takes the form like below: 
\begin{eqnarray}\label{eqn202302150323}
F_M (\mathbf{x}_i )=\gamma_0+\gamma\sum_{m=1}^M f_m(\mathbf{x}_i),f_m\in \mathcal{F}, 
\end{eqnarray}
where $M$ is the size of the ensemble and $\mathcal{F}$ is a collection of decision trees. We assume each base learner $f_m$ is a function from $\mathcal{F}$. $\gamma$ is a learning rate. Number of trees and learning rate are both hyperparameters need to be tuned. To learn the set of trees $f_m$  used in the model, a loss criterion $L(F_M (\mathbf{x}_i ),y_i)$ should be specified. For data with continues response, $L$ is specified to be squared error that $L(F_M (\mathbf{x}_i ),y_i)= \{y_i-F_M (\mathbf{x}_i )\}^2$. To estimate model parameters, we minimize the average loss: 
\begin{eqnarray}
    \mathcal{E}_{train}(F_M)=\frac{1}{n}\sum_{i=1}^n L(F_M(\mathbf{x}_i,y_i)).
\end{eqnarray}

GBDT begins with initializing the model with a particular constant $\gamma_0$, which can be either treated as a hyperparameter or determined by minimizing loss function. Next, a tree $f_1$ is constructed to minimize the loss function from the training sample. To reduce overfitting, the tree is modified by multiplying a parameter $\gamma\in[0,1]$ (usually called learning rate). Model is updated by adding the new tree $(\gamma f_1)$ to the initial model ($\gamma_0$). The next iteration of the algorithm begins by constructing a new tree ($f_2$) from the training set $(\{\mathbf{x}_i \}_{i=1}^n)$ and the residual from the previous iteration $(\{y_i-\gamma_0-\gamma f_1 (\mathbf{x}_i ) \}_{i=1}^n )$. This new regression tree (after updating by the learning rate) plus the model from previous iteration make up the updated model, based on the second iteration. This process is repeated for each subsequent iteration until $M$ trees are constructed.

A GBDT model can be expressed as a linear model when we one hot encode each leaf in the trees to one dummy variable. To see this, we firstly give a review of how decision tree works. A decision tree partitions the feature space into a set of disjoint rectangles and fits a constant at each region. Let $\mathbf{1}_A$ be the indicator function on region $A$ that $\mathbf{1}_A (\mathbf{x}_i )=1$ when $\mathbf{x}_i \in A$ and $0$ otherwise. A decision tree can be expressed as
\begin{eqnarray}
f_m (\mathbf{x}_i )=\sum_{j=1}^{J_m}b_{jm}\mathbf{1}_{R_{jm}}(\mathbf{x}_i).
\end{eqnarray}
The $m$-th tree partitions the input space into $J_m$ disjoint regions $R_1,R_2,\cdots,R_{jm}$ and predicts a constant value ($b_{jm}$) in each region. We can further rewrite Equation (\ref{eqn202302150323}) as 
\begin{eqnarray}\label{eqn202302150324}
F_M (\mathbf{x}_i )=\gamma_0+\gamma\sum_{m=1}^M\sum_{j=1}^{J_m}b_{jm} \mathbf{1}_{R_{jm}}(\mathbf{x}_i ) .
\end{eqnarray}
Equation (\ref{eqn202302150324}) indicates that the final model $F_M (\mathbf{x}_i )$  can be viewed as a linear model, with each $\mathbf{1}_{R_{jm}}  (\mathbf{x}_i )$ as a feature (dummy variable). The coefficient of each feature is the leaf ($b_{jm}$) value multiplied by the learning rate ($\gamma$).  Prediction value of the final model for each sample $\mathbf{x}_i $ is sum of each dummy variable ($\mathbf{1}_{R_{jm}} (\mathbf{x}_i )$) multiplied by its coefficient ($\gamma b_{jm}$), plus the intercept ($\gamma_0$).  We show an illustration of GBDT OHE in Figure \ref{fig202303132031}. 
\begin{figure}
\centering
\includegraphics[width=1\textwidth,height=0.4\textwidth]{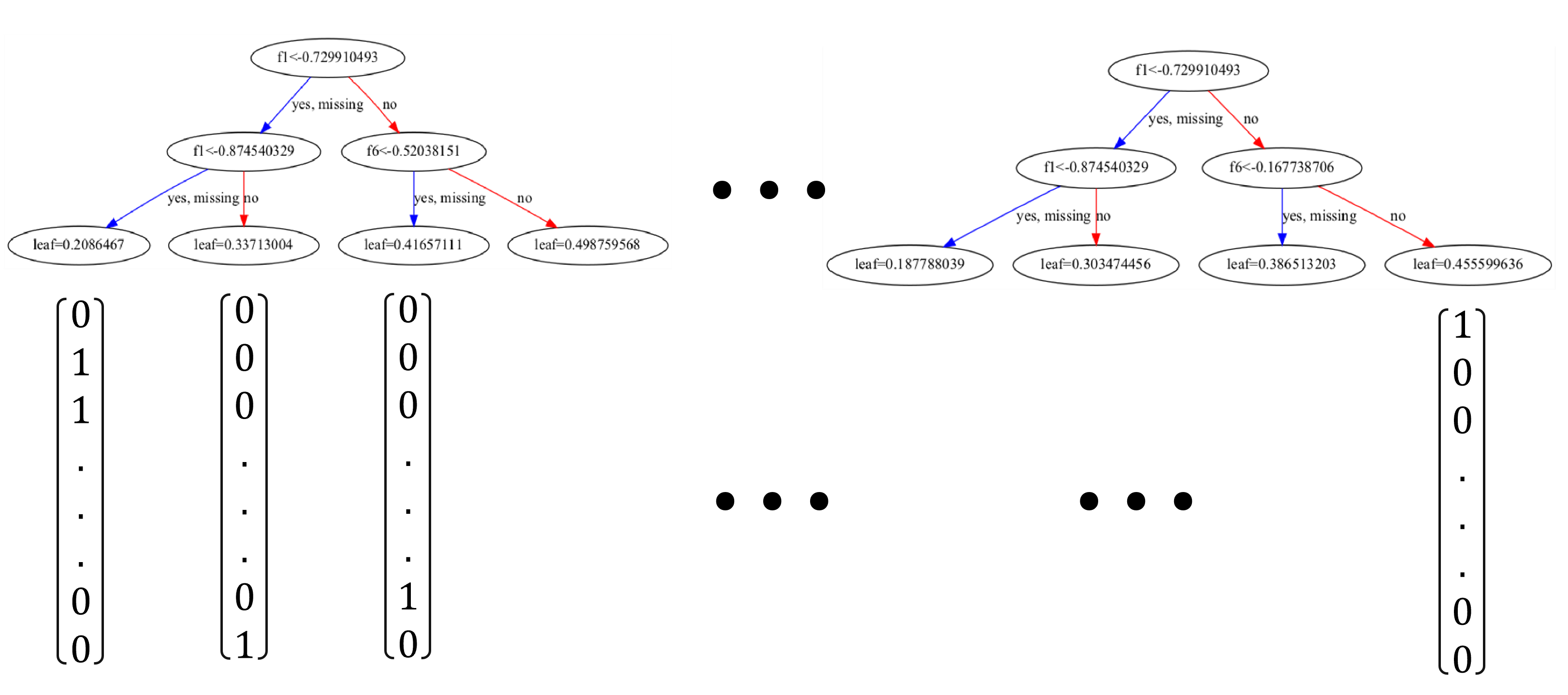}
\caption{One-hot encoding leaves of a specific GBDT model. Each leaf in each tree is one-hot encoded into a dummy variable.}\label{fig202303132031}
\end{figure}

This representation is similar with RuleFit proposed by \citet{friedman2008predictive} but not the same. In RuleFit, each node in a decision tree (including both leaf nodes and non-leaf nodes) is one-hot encoded as a dummy variable. Our method is distinguished from RuleFit in that we only one-hot encode the leaves of the gradient boosting trees into dummy variables, rather than all nodes. This modification is reasonable because the prediction of each gradient-boosted tree can be viewed as a linear combination of only the leaf nodes.

There are different ways to implement GBDT models, such as XGBoost proposed by \citet{chen2016xgboost}, LightGBM proposed by \citet{ke2017lightgbm}, and CatBoost proposed by \citet{prokhorenkova2018catboost}. All of them can be one-hot encoded using the method described in Equation (\ref{eqn202302150324}). In our study, we used XGBoost to implement GBDT models, specifically the xgboost module in Python, but this approach can be easily generalized to other implementation methods.

\subsection{Risk Decomposition}\label{subsec202303271457}
We further rewrite Equation (\ref{eqn202302150324}) as 
\begin{eqnarray}\label{eqn202303131234}
F_M (\mathbf{x}_i )=\sum_{k=0}^p b_k \phi_k(\mathbf{x}_i)=\Phi^\top(\mathbf{x}_i)\beta.
\end{eqnarray}
Each $\phi_k, k > 0$ is an indicator function of one unique leaf in the GBDT model and $\phi_0$ is a vector of all ones representing the constant term. $\Phi^\top(\mathbf{x}_i)=(\phi_0(\mathbf{x}_i),\phi_1(\mathbf{x}_i),\cdots,\phi_p(\mathbf{x}_i))^\top$ and $\beta=(b_0,b_1,\cdots,b_p)^\top.$ 

Now we treat GBDT as a feature engineering process and $\Phi(\mathbf{x})$ is the vector of features generated. Suppose $\hat\beta$ is the estimator of $\beta$.  $\hat\beta$ can be contributions of leaves in an XGBoost model or linear regression refitting coefficients with regularization. To see the robustness of the model, we suppose $\Delta\Phi$ is a random perturbation term that has zero mean and is independent of $\Phi$.  We can decompose the risk function as the following:
\begin{eqnarray}
&&E_\mathcal{X}E_\mathcal{D}(y-(\Phi^T(\mathbf{x})+\Delta\Phi^\top)\hat{\beta}))^2\nonumber\\
&=&E_\mathcal{X}E_\mathcal{D}(f(\mathbf{x})-(\Phi^T(\mathbf{x})\hat{\beta}))^2+E_\mathcal{X}E_\mathcal{D}(\Delta\Phi^T\hat\beta)^2+\sigma_\varepsilon^2\nonumber\\
&=&E_\mathcal{X}[f(\mathbf{x})-\Phi^\top(\mathbf{x})E_{\mathcal{D}}(\hat\beta)]^2+E_\mathcal{X}E_\mathcal{D}[\Phi^\top(\mathbf{x})E_{\mathcal{D}}(\hat\beta)-\Phi^\top(\mathbf{x})\hat\beta]^2+E_\mathcal{X}E_\mathcal{D}(\Delta\Phi^T\hat\beta)^2+\sigma_\varepsilon^2
\nonumber\\
&=:&(i)+(ii)+(iii)+(iv),\label{eqn202303241602}
\end{eqnarray}
where $(i)$ is model bias$^2$, $(ii)$ is model variance,
$(iii)$ is perturbation term, and $(iv)$ corresponds to  irreducible error term.
\begin{figure}
\centering
\includegraphics[width=0.5\textwidth,height=0.5\textwidth]{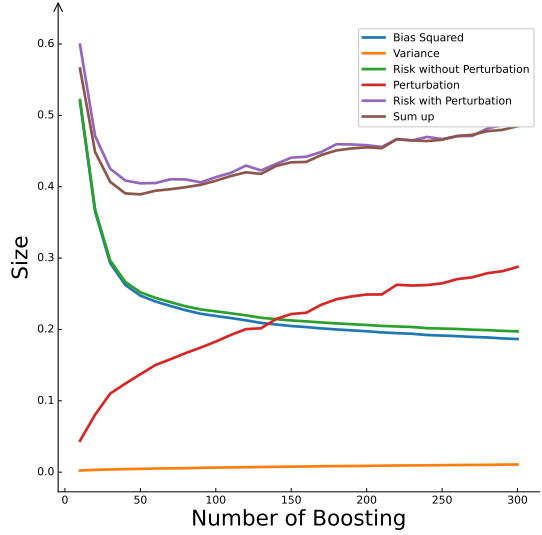}
\caption{Decomposition of risk function with perturbation.}\label{fig202303230220}
\end{figure}

Figure \ref{fig202303230220} provides an illustration of our proposed method for decomposing the risk function. The curve labeled ``Sum up"
 is the sum of curves $(i)-(iv)$, while the curve labeled ``Risk with Perturbation" corresponds to the risk calculated directly from the sample. As shown in Figure \ref{fig202303230220}, there is a small gap between the two curves, which is due to the violation of the independent assumptions made on the perturbation term. However, the difference between the two curves is negligible, and thus we can conclude that by decomposing the perturbation term, we are able to split the perturbation term out and effectively evaluate the robustness of a GBDT model.

\section{Ehnaching Robustness through Regularization}\label{sec202303271438}
\subsection{GBDT OHE upon Regularization}
In one-hot encoded GBDT, the number of leaves ($p$ in Equation (\ref{eqn202303131234})) can be similar to or much larger than the number of samples ($N$). Due to the number of leaves in GBDT, We are dealing with a very high dimensional model fitting problem, where  a large number of features are  created to reduce possible modeling bias. However, when the feature space is high-dimensional, a naive least squares estimator is prone to high variance, making regularization essential.  
The most well-known regularized estimators are based on the regularized $M$-estimation including the Ridge \citep{hoerl1970ridge} and the Lasso \citep{tibshirani1996regression}. 
These methods have attracted a large amount of theoretical and algorithmic studies. These methods can improve the accuracy of parameter estimates and prediction performance when feature space is large. Meanwhile, in this section, we show that regularization can improve not only the performance but also the robustness of the original GBDT model. We demonstrate how we refit a GBDT model with one-hot encoding and regularization and discuss its theoretical benefits on both robustness and performance.

 Suppose there are $p$ leaves and the coefficient of the $i$-th leaf is $b_i$. For each leaf, we one-hot encode it into a feature $\mathbf{1}_{R_k}$ and $\mathbf{1}_{R_k}(\mathbf{x}_i )=1$ when the $i$-th sample is on this leaf and 0 otherwise. We apply regularization methods including the Ridge and Lasso. The followings are the formulas of these methods.\\
Ridge regression estimator is defined as 
\begin{eqnarray}
    \{b_k\}_{k=0}^p=\arg\min_{\{b_k\}_{k=0}^p} \left\{\sum_{i=1}^n L\left(y_i,b_0+\sum_{k=1}^p b_k \mathbf{1}_{R_k}(X_i)\right)+\lambda\sum_{k=1}^p b_k^2\right\}.
\end{eqnarray}
Ridge regularization is a form of $L_2$ regularization that adds a penalty term to the objective function that is proportional to the sum of the squares of the weights.\\
Lasso estimator is defined to be
\begin{eqnarray}
    \{b_k\}_{k=0}^p=\arg\min_{\{b_k\}_{k=0}^p} \left\{\sum_{i=1}^n L\left(y_i,b_0+\sum_{k=1}^p b_k \mathbf{1}_{R_k}(X_i)\right)+\lambda\sum_{k=1}^p |b_k|\right\}.
\end{eqnarray}
Lasso regularization is a form of $L_1$ regularization that adds a penalty term to the objective function that is proportional to the sum of the absolute values of the weights.

Here $\{b_k\}_{k=1}^p$ represents the new refitted values of leaves, while $b_0$ denotes the common constant.

We outline our methodology and workflow for fitting a GBDT one-hot encoding model in Algorithm \ref{alg1}.

\begin{algorithm}[H]
\caption{}\label{alg1}
\begin{algorithmic}
\item[Step 1.] Find a well-tuned GBDT model (E.g. XGBoost model) to generate leaves. 
 \item [Step 2.] One-hot encode the GBDT model and delete the same leaves.
\item [Step 3.] Using the regularization method in linear regression to refit the coefficients. Apply regularization methods: Ridge ($L_2$), Lasso ($L_1$).
\end{algorithmic}
\end{algorithm}
Remark: Step 1 involves selecting a well-tuned GBDT model. For example, it can be done through randomized grid search with cross-validation.

\subsection{Regularization Effect on Robustness and Performance}
Risk decomposition in Section \ref{subsec202303271457} motivates us to apply regularization to control the perturbation term. Specifically, we may use $L_2$ regularization to control the perturbation term $(iii)$. Additionally, \citet{xu2008robust} gave evidence of using $L_1$ regularization for perturbation to gain robustness.
Instead of treating the perturbation term as a random term as we did in Section \ref{sec202302161144}, we can also consider a robust optimization problem under a fixed design scenario.
\begin{eqnarray}\label{eqn202302161207}
\min _{\beta \in \mathbb{R}^p}\left\{\max _{\Delta \Phi \in \mathcal{U}}\|y-(\Phi+\Delta \Phi) \beta\|_2\right\},
\end{eqnarray}
with
\begin{eqnarray}\label{eqn202302161208}
\mathcal{U} \triangleq\left\{\left(\Delta\Phi_1, \cdots, \Delta\Phi_p\right) \mid\left\|\Delta\Phi_i\right\|_2 \leq c_i, \quad i=1, \cdots, p\right\}.
\end{eqnarray}
\citet{xu2008robust} showed that the robust regression problem (\ref{eqn202302161207}) with uncertainty set of the form (\ref{eqn202302161208}) is equivalent to the $L_1$ regularized regression problem (\ref{eqn202303271316}).
\begin{eqnarray}
\min _{\beta \in \mathbb{R}^p}\left\{\|y-\Phi\beta \|_2+ c_i|\beta|\right\}\label{eqn202303271316}.
\end{eqnarray}
It is worth noting that if we modify the uncertainty set (\ref{eqn202302161208}) to
\begin{eqnarray}\label{eqn202302161430}
    \mathcal{U}_1 \triangleq\left\{\Delta \Phi \mid\left\|\Delta\Phi\right\|_2 \leq c\|\beta\|_2\right\},
\end{eqnarray}
or 
\begin{eqnarray}\label{eqn202302191008}
    \mathcal{U}_2 \triangleq\left\{\left(\Delta\Phi_1, \cdots, \Delta\Phi_p\right) \mid\left\|\Delta\Phi_i\right\|_2 \leq c|\beta_i|, \quad i=1, \cdots, p\right\},
\end{eqnarray}
 we can show the robust regression problem is equivalent to a $L_2$ regularized regression problem. Here  $\|A\|_2$ is the largest singular of a matrix $A$. $\|A\|_2$  boils down to $l_2$ norm of $A$ if $A$ is a vector. (\ref{eqn202302161430}) and (\ref{eqn202302191008}) depicts a case when perturbation size is always controlled by the signal level. The result is shown in Theorem 1.

\begin{theorem}
The robust regression problem (\ref{eqn202302161207}) with uncertainty set of the form (\ref{eqn202302161430}) or (\ref{eqn202302191008}) is equivalent to the following $L_2$ regularized regression problem:
\begin{eqnarray}
\min _{\beta \in \mathbb{R}^p}\left\{\|y-\Phi\beta \|_2+ c\|\beta\|^2_2\right\} .
\end{eqnarray}
\end{theorem}
\begin{proof}
We show the proof when the uncertainty set is  (\ref{eqn202302161430}). The proof when uncertainty set is (\ref{eqn202302191008}) is similar to proof of Theorem 1 in \citet{xu2008robust}.
\begin{eqnarray*}
&&\max _{\Delta \Phi \in \mathcal{U}}\|y-(\Phi+\Delta \Phi) \beta\|_2\\
&\leq&\|y-\Phi\beta\|_2+\max _{\Delta \Phi \in \mathcal{U}}\|\Delta\Phi\beta\|_2\\
&\leq&\|y-\Phi\beta\|_2+\max _{\Delta \Phi \in \mathcal{U}}\|\Delta\Phi\|_2\|\beta\|_2\\
&=&\|y-\Phi\beta \|_2+ c\|\beta\|^2_2
\end{eqnarray*}

Now, let
$\mathbf{u}=\dfrac{y-\Phi \beta}{\left\|y-\Phi \beta\right\|_2},$ and let
$\Delta\Phi^*_i\triangleq-c \beta_i \operatorname{sgn}\left(\beta_i\right) \mathbf{u},$ we have $\|\Delta\Phi^*\|_2=c\|\beta\|_2,$ therefore $\Delta\Phi^*\in\mathcal{U}$. Meanwhile,
\begin{eqnarray*}
&&\max _{\Delta \Phi \in \mathcal{U}}\|y-(\Phi+\Delta \Phi) \beta\|_2\\
&\geq&\max _{\Delta \Phi \in \mathcal{U}}\|y-(\Phi+\Delta \Phi^*) \beta\|_2\\
&=&\left\|y-\Phi\beta+c\|\beta\|_2^2 \mathbf{u}\right\|_2\\
&=&\|y-\Phi\beta \|_2+ c\|\beta\|^2_2.
\end{eqnarray*}
This completes the proof of the theorem.
\end{proof}

Here we do not make a definitive conclusion on which regularization method is better. The theoretical results are based on different assumptions that are usually hard to verify in practice. Instead, we emphasize that applying any form of regularization can improve the model robustness. In the later of this section and in  Section \ref{sec202303131258}, we present numerical results for both $L_1$ and $L_2$ regularization, which demonstrate that both regularization methods can effectively improve model robustness.

In practice, our goal is often to improve the robustness of the model without sacrificing too much performance. In addition, we also aim to evaluate the impact of refitting and regularization using GBDT OHE on model performance, as represented by the risk before perturbation. To guide our theoretical analysis, we propose several assumptions.

\emph{Assumption 1}:
$f(\mathbf{x})-E_\mathcal{X}(y|\Phi(\mathbf{x}))$ is uncorrelated with $\Phi(\mathbf{x})$.

Remark: We trivially have that $f(\mathbf{x})-E_\mathcal{X}(y|\Phi(\mathbf{x}))$ is uncorrelated with $E_\mathcal{X}(y|\Phi(\mathbf{x}))$, which is a function of $\Phi(\mathbf{x})$. Assumption 1 is a stronger assumption.

\emph{Assumption 2}:
$E(y|\Phi(\mathbf{x}))=\Phi^\top(\mathbf{x})\beta_0$.

Noticing that under Assumption 1, the model bias term $(i)$ can be further decomposed to 
\begin{eqnarray}
&&E_\mathcal{X}[f(\mathbf{x})-\Phi^\top(\mathbf{x})E_{\mathcal{D}}(\hat\beta)]^2\nonumber\\
&=&E_\mathcal{X}[f(\mathbf{x})-E_\mathcal{X}(y|\Phi(x)]^2+E_\mathcal{X}[E_\mathcal{X}(y|\Phi(x))-\Phi^\top(\mathbf{x})E_{\mathcal{D}}(\hat\beta)]^2\nonumber\\
&=:&(v)+(vi)\label{eqn202303241640},
\end{eqnarray}
where $(v)$ is the model misspecification bias and $(vi)$ is the model in-class bias.

Under Assumption 2,  we can further have
$$f(\mathbf{x})-E_\mathcal{X}(y|\Phi(x))=f(\mathbf{x})-\Phi^\top(\mathbf{x})\beta_0$$
and 
$$E_\mathcal{X}(y|\Phi(\mathbf{x}))-\Phi^\top(\mathbf{x})E_{\mathcal{D}}(\hat\beta)=\Phi^\top(\mathbf{x})(E_{\mathcal{D}}(\hat\beta)-\beta_0).$$

\begin{figure}
\centering
\includegraphics[width=0.5\textwidth,height=0.5\textwidth]{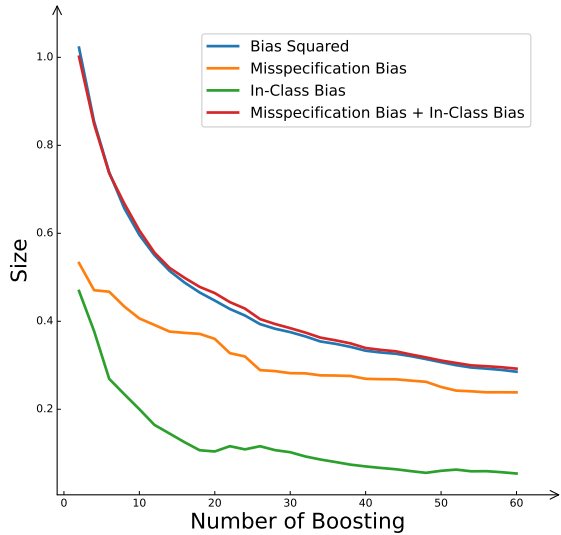}
\caption{Bias term decomposition on the simulation data.}\label{fig202303131316}
\end{figure}

In a GBDT model, the model misspecification bias decreases with increasing boosting number, but the in-class bias is uncertain due to the unknown best coefficient $\beta_0$. Figure \ref{fig202303131316} illustrates the decomposition of the GBDT model fitting on a simple regression model $y=x^2+\varepsilon, x \text{ and } \varepsilon\sim N(0,1)$, under some small boosting numbers, where a Ridge regression refitting of the coefficients is treated as the true coefficients. One-hot encoding a GBDT model does not change $\Phi(\bf{x})$, so the model misspecification bias remains unchanged when doing refitting. The regularization size affects the in-class bias and variance term of the GBDT one-hot encoding model. We need to point out that when the boosting numbers and the number of leaves are large, the Ridge estimator is no longer the best solution.  In Section \ref{sec202303131258}, we show that with the numerical study regularized GBDT one-hot encoding can decrease the in-class bias and variance terms of the original GBDT model with a relatively small regularization size. Additionally, the perturbation term can be decreased with a slight sacrifice of the in-class bias and variance terms.

To conclude this section, we present the risk decomposition results for Lasso ($L_1$) and Ridge ($L_2$) regularization using different regularization sizes. Specifically, we fit an XGBoost model to CHP data, apply one-hot encoding, and then apply $L_1$ and $L_2$ regularization, respectively. To evaluate the model's robustness, we add a 5\% perturbation to the testing data. Figure \ref{fig202303131938} shows the results. In the Figure, curve labelled "Sum up" shows the sum of the bias squared term, variance term, and perturbation term. Theoretically, it should be the same as the risk. The small gap between the risk curve and the sum up curve is due to the violation of the independence assumption between perturbation and data.  As shown in Figure \ref{fig202303131938}, the risk function has a U-shape, with the perturbation term decreasing as the regularization parameter size increases. It indicates that larger perturbation helps improve the model's robustness. However, the bias of the model will increase when the regularization parameter is large. In practice, we need to find a trade-off point to minimize the risk function.

\begin{figure}
	\centering
	\subfigure{
		\begin{minipage}[b]{0.4\textwidth}
			\includegraphics[width=1\textwidth,height=7cm]{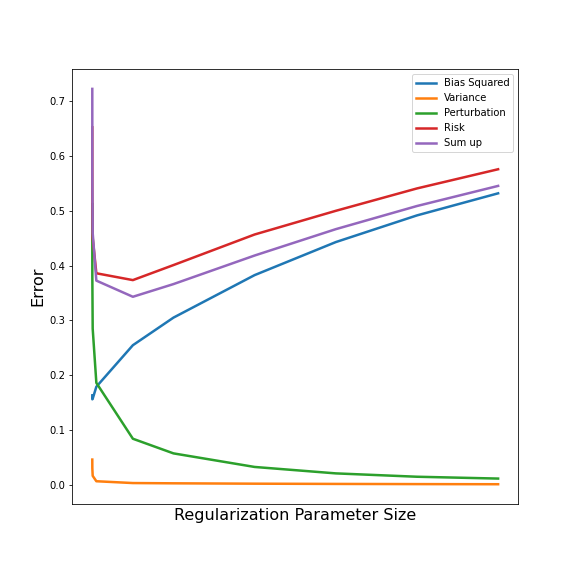}
		\end{minipage}
	   }
    	\subfigure{
    		\begin{minipage}[b]{0.4\textwidth}
   		 	\includegraphics[width=1\textwidth,height=7cm]{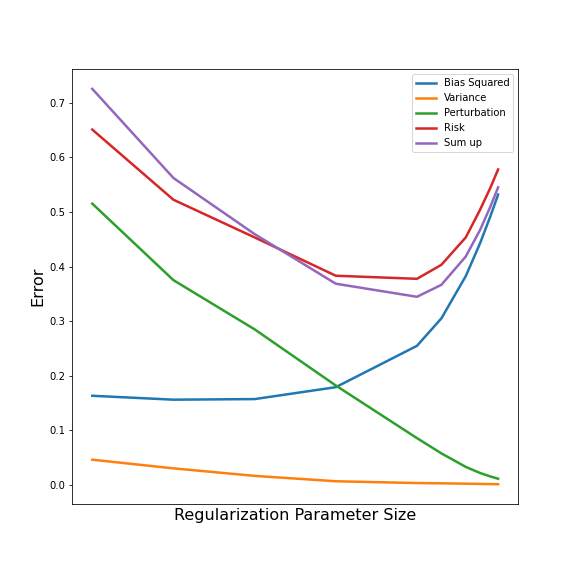}
    		\end{minipage}
    	}
    	\caption{Left: GBDT one-hot encoding with Lasso regularization. Right:  GBDT one-hot encoding with Ridge regularization. }\label{fig202303131938}
\end{figure}

\section{Numerical Study}\label{sec202303131258}
In this section, we show the performance and robustness of different methods using several real data sets. We use XGBoost for training gradient boosting models. Proper selection of hyperparameters is crucial for developing a robust XGBoost model. In this paper, we focus on several key hyperparameters that are known to have a significant impact on the final results on XGBoost models. These include n\_estimators (number of trees in the XGB model), max\_depth (maximum depth of each tree), learning\_rate (step size at which the algorithm makes updates to the XGBoost parameters), gamma (minimum loss reduction required for a further partition on a leaf node of a tree), reg\_alpha (hyperparameter for $L_1$ regularization) and reg\_lambda (hyperparameter for $L_2$ regularization).  In XGBoost, these hyperparameters can be classified into two types. The parameters n\_estimators, max\_depth, learning\_rate are related to the overall structure and behavior of an XGBoost model, while gamma, reg\_lambda, reg\_alpha are more related to regularization and model complexity. 
\begin{figure}[!]
	\centering
	  {(a) Test MSEs for Airfoil Data}\\
   \vskip-0.75cm
 \subfigure{
		\begin{minipage}[b]{0.6\textwidth}
			\includegraphics[height=9.4cm]{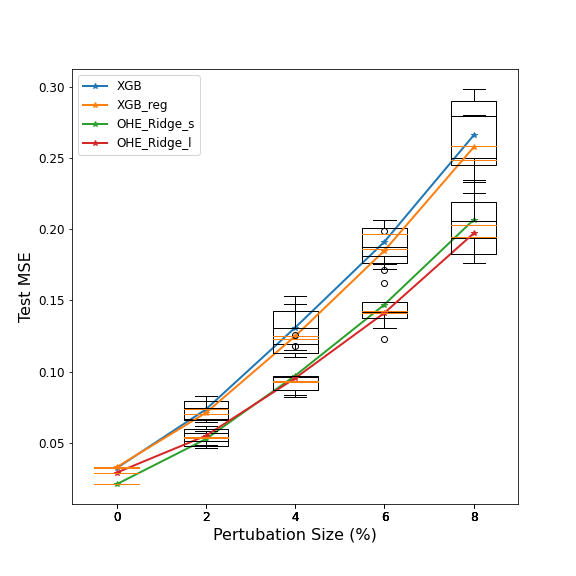}
		\end{minipage}
	   }\\
        {(b) Test MSE for CHP Data}\\
   \vskip-0.75cm
     \subfigure{
    		\begin{minipage}[b]{0.6\textwidth}
   		 	\includegraphics[height=9.4cm]{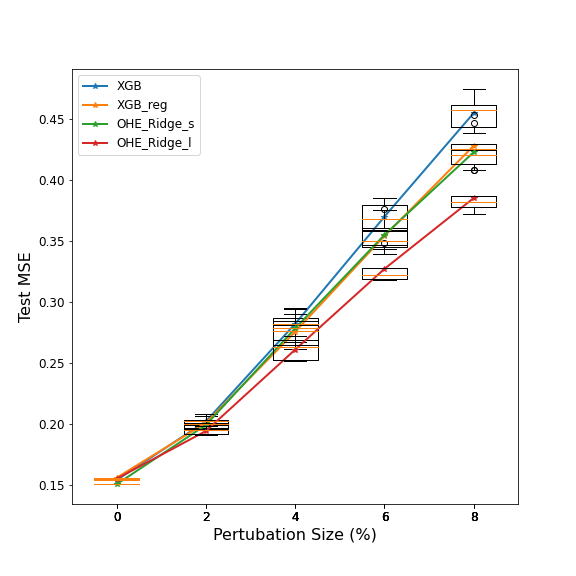}
    		\end{minipage}
    	}
    	\caption{Test MSE under different perturbation sizes. (a) is for Airfoil data, and (b) is for CHP data.}\label{fig202303241345}
\end{figure}

To create a well-tuned XGBoost model, we randomly split the raw data into training (80\%) and testing (20\%) sets, and used 5-fold cross-validation and randomized search to optimize the hyperparameters. To explore the effect of regularization parameters on the robustness of XGBoost, we trained two XGBoost models: one with  n\_estimators, max\_depth, learning\_rate tuned and with all regularization parameters set to 0 (XGB), and another with gamma, reg\_lambda, and reg\_alpha also optimized (XGB\_reg). We further fit GBDT OHE models by applying Ridge and Lasso regularization with different sizes (small, medium, and large) to the well-tuned XGBoost models. These models are denoted as OHE\_Ridge\_s, OHE\_Ridge\_m, OHE\_Ridge\_l, OHE\_Lasso\_s, OHE\_Lasso\_m, and OHE\_Lasso\_l. Final tuned parameters and model settings can be found in Appendix I.

\begin{table}\caption{Test MSE/Perturbation term of different methodologies}\label{tab202303241604}
\centering
\begin{tabular}{cccccc}\toprule
            &Airfoil(0\%)& Airfoil(2\%) & Airfoil(5\%)       \\\hline
XGB         & 0.032/0    & 0.074/0.046  & 0.156/0.134        \\
XGB\_reg    & 0.033/0    & 0.071/0.039  & 0.153/0.124        \\ \hline
OHE\_Ridge\_s    &0.020/0     & 0.053/0.032  & 0.120/0.099     \\ 
OHE\_Ridge\_m    & 0.021/0    & 0.052/0.029  & 0.119/0.092     \\
OHE\_Ridge\_l    & 0.029/0    & 0.054/0.025  & 0.117/0.083      \\\hline
OHE\_Lasso\_s    & 0.022/0    & 0.058/0.036  & 0.125/0.100     \\
OHE\_Lasso\_m    & 0.025/0    & 0.055/0.033  & 0.121/0.098       \\
OHE\_Lasso\_l    &0.026/0     & 0.056/0.031  & 0.120/0.096   \\
%
\toprule
             &CHP(0\%)& CHP(2\%)    &CHP(5\%)    \\\hline
XGB            &0.154/0& 0.202/0.058  & 0.324/0.180   \\
XGB\_reg      &0.155/0& 0.201/0.053    & 0.316/0.168     \\ \hline
OHE\_Ridge\_s      &0.151/0& 0.199/0.057 &0.318/0.173 \\ 
OHE\_Ridge\_m       &0.155/0& 0.194/0.039 & 0.295/0.131    \\
OHE\_Ridge\_l     &0.170/0& 0.201/0.028   & 0.287/0.102    \\\hline
OHE\_Lasso\_s       &0.151/0& 0.205/0.062  & 0.331/0.186    \\
OHE\_Lasso\_m     &0.153/0& 0.201/0.053  & 0.317/0.164 \\
OHE\_Lasso\_l    &0.179/0& 0.211/0.039  &0.305/0.125   \\
%
\toprule
            & BS(0\%) & BS(5\%) &BS(10\%)      &      \\\hline
XGB            & 0.159/0  & 0.215/0.059   &0.349/0.198     \\
XGB\_reg        & 0.160/0  & 0.212/0.057  &0.343/0.196      \\ \hline
OHE\_Ridge\_s         & 0.158/0  & 0.212/0.057   &0.345/0.197   \\ 
OHE\_Ridge\_m     & 0.159/0  & 0.212/0.053   &0.343/0.187   \\
OHE\_Ridge\_l       & 0.161/0  & 0.213/0.051   &0.342/0.181     \\\hline
OHE\_Lasso\_s       & 0.159/0  & 0.213/0.059  &0.349/0.200      \\
OHE\_Lasso\_m       & 0.158/0  & 0.212/0.056  &0.346/0.196     \\
OHE\_Lasso\_l       & 0.159/0  & 0.213/0.055  &0.346/0.193     \\\toprule
\end{tabular}
\end{table}

To evaluate the robustness of different models, we consider different perturbation sizes on the testing data and compare their performance. Specifically, we perturb each predictor $X_j$ by adding a normal distributed perturbation $\delta_j$ of size $\sigma$, resulting in a new predictor $\tilde{X}_j=X_j+\delta_j$. We then measure the test mean squared error (Test MSE) and the decomposed Perturbation term of each model under different perturbation sizes. Perturbation term is computed as a sample version of (iii) in Equation (\ref{eqn202303241602}). To account for randomness, we repeat the perturbation process five times for each size. We show the variability of the results on Airfoil and CHP data in Figures \ref{fig202303241345}. We show detailed Test MSE/Perturbation term of different models under different sizes of perturbation in Table \ref{tab202303241604}. A detailed introduction of the data sets we used can be found in Appendix II.

We have several observations based on our numerical study:
\begin{itemize}
    \item XGB\_reg can slightly improve the robustness of XGB when there is a perturbation. It can perform better than XGB on Airfoil (5\%), CHP(5\%), and BS(10\%), suggesting that models with no regularization are more sensitive to data changes.
    \item GBDT OHE  models with relatively small regularization sizes can keep or improve the performance of the tuned XGBoost models XGB or XGB\_reg. Specifically, the Test MSE of OHE\_Ridge\_s and OHE\_Lasso\_s on data set with no perturbation (0\%)  can generally keep or decrease the Test MSE compared with XGB and XGB\_reg when there is no perturbation. It is due to that refitting and regularization can decrease the model bias, especially the in-class bias term.
    \item GBDT OHE models with larger regularization tend to have worse performance than those with smaller regularization, mainly due to the increasing model bias. However, GBDT one-hot encoding models with relatively larger regularization sizes can provide more robust predictions. As the regularization size increases, the perturbation term decreases, which suggests that models with large regularization sizes are more robust for scenarios with larger perturbation sizes on unseen data.
    \item GBDT OHE with larger regularization sizes, the Test MSE grows slower compared with that with smaller regularization sizes. It can be seen in Figure \ref{fig202303241345}. 
    \item Compared to the regularized XGBoosting model XGB\_reg, GBDT OHE models with regularization typically exhibit better robustness in the presence of data perturbation. It can be found by that GBDT OHE models with larger regularization sizes always have better performance than XGB\_reg when there is a perturbation. 
\end{itemize}

\section{Concluding Remarks and Future Work}\label{sec202303291157}
In this study, we first demonstrate that traditional GBDT models are becoming less robust against perturbations on unseen data with number of boosting increasing. To address this issue, we propose a GBDT OHE representation that transforms a GBDT model into a linear framework, allowing for the use of a novel risk decomposition to analyze model robustness. We further introduce GBDT OHE with regularization to refit and regularize a GBDT model, and provide both theoretical and numerical results of its impact on model robustness and performance. Our study puts emphasis on the importance of evaluating model robustness and demonstrates the effectiveness of one-hot encoding with regularization in enhancing model robustness, which can have practical implications for developing more reliable and accurate machine learning models in various domains.

There are several directions for future work that can further enhance the applicability and performance of the method. In this paper, we only give theoretical evidence of regularization for regression data set with continuous response. One promising area is to investigate the theory of the method of classification data set. Additionally, exploring alternative regularization techniques that might further improve the robustness and performance of the one-hot encoded GBDT models is also a potential topic.

Another important direction for future work is to investigate the robustness of machine learning models beyond GBDT. While GBDT can be one-hot encoded into a linear model due to its property of splitting the data space into non-overlapping regions with constant predictions, not all machine learning models possess such a property. However, some models do have a simple local model, such as the ReLu neural network which has a local linear model \citep{sudjianto2020unwrapping}. An open problem is whether we can generalize our one-hot encoding method to machine learning methods with simple local models, to analyze and enhance their robustness.

\section*{Appendix}\label{sec202303291524}
\subsection*{Appendix I: Data Sets}
The following introduces the data sets we used for real data analysis.
\begin{itemize}
    \item[1.]	Airfoil. Airfoil dataset is a collection of data obtained from NASA’s aerodynamic and acoustic tests of two and three-dimensional airfoil blade sections. It includes measurements from various NACA 0012 airfoils at different wind tunnel speeds and angles of attack. The dataset contains 1503 observations and 6 variables, with the response variable being the scaled sound pressure level in decibels. The five independent variables include frequency in Hertz, angle of attack in degrees, chord length in meters, free-stream velocity in meters per second, and suction side displacement thickness in meters.

\item[2.] California Housing (CHP). CHP dataset is derived from the 1990 U.S. census and includes information on median house values for California districts. It is based on data from block groups, which are the smallest geographical units for which the U.S. Census Bureau publishes sample data. The dataset includes 20640 observations and the independent variables include longitude, latitude, housing median age, medium income, population, total rooms, total bedrooms and households.

\item[3.] Bike Sharing (BS). This dataset is collected from Capital bike share system. It contains information on the hourly and daily count of rental bikes between years 2011 and 2012. It has 17379 data points and originally had 16 predictors. We removed 2 non-relevant and 2 response-related variables, leaving us with 12 variables including hour, temperature, feel temperature, humidity, wind speed, season, working day, weekday, weather situation, year and month. 
\end{itemize}

\subsection*{Appendix II: Model Tuning Results}
Here are the final tuning parameter results of different models fitted in Section~\ref{sec202303131258}.
\begin{table}[h]
\centering
\begin{tabular}{cccc}\toprule
                                                  & Airfoil     &CHP & BS       \\\hline
XGB                    & \begin{tabular}[c]{@{}l@{}}max\_depth: 7\\ n\_estimators: 500\\learning\_rate: 0.15\end{tabular} 
& \begin{tabular}[c]{@{}l@{}}max\_depth: 6\\ n\_estimators: 600\\learning\_rate: 0.1\end{tabular}  
& \begin{tabular}[c]{@{}l@{}}max\_depth=6\\ n\_estimators:500\\learning\_rate: 0.07\end{tabular} \\\hline
XGB\_reg               & \begin{tabular}[c]{@{}l@{}}reg\_alpha: 0.3\\ reg\_lambda: 0.2\\ gamma: 0\end{tabular} & \begin{tabular}[c]{@{}l@{}}reg\_alpha: 1.75\\ reg\_lambda: 1\\ gamma: 0\end{tabular}   & \begin{tabular}[c]{@{}l@{}}reg\_alpha: 1.55\\ reg\_lambda: 0.5\\ gamma: 0.07\end{tabular}      \\ \hline
Ridge\_s&$\lambda_2=0.1$&$\lambda_2=100$&$\lambda_2=200$\\
Ridge\_m&$\lambda_2=1$&$\lambda_2=400$&$\lambda_2=400$\\
Ridge\_l&$\lambda_2=10$&$\lambda_2=1000$&$\lambda_2=600$\\\hline
Lasso\_s&$\lambda_1=6e^{-5}$&$\lambda_1=4e^{-4}$&$\lambda_1=2e^{-4}$\\
Lasso\_m&$\lambda_1=8e^{-5}$&$\lambda_1=7e^{-4}$&$\lambda_1=3e^{-4}$\\
Lasso\_l&$\lambda_1=10e^{-5}$&$\lambda_1=2e^{-3}$&$\lambda_1=4e^{-4}$\\\toprule
\end{tabular}
\end{table}

\bibliographystyle{apalike}
\bibliography{main}

\end{document}